\documentclass[11pt]{llncs}
\usepackage{amsmath,amssymb}

\usepackage{url}
\usepackage{fullpage}

\begin{document}
\title{Approximating  solution structure of the Weighted Sentence Alignment problem} 
\author{Antonina Kolokolova \and Renesa Nizamee} 
\institute{Memorial University of Newfoundland \\ 
\{kol,mrn271\}@mun.ca} 

\maketitle
\begin{abstract}

We study the complexity of approximating solution structure of the bijective weighted sentence alignment problem of DeNero and Klein (2008). In particular, we consider the complexity of finding an alignment that has a significant overlap with an optimal alignment. 

We discuss ways of representing the solution for the general weighted sentence alignment as well as phrases-to-words alignment problem, and show that computing a string which agrees with the optimal sentence partition on more than half (plus an arbitrarily small polynomial fraction) positions for the phrases-to-words alignment is NP-hard. For the general weighted sentence alignment we obtain such bound from the agreement on a little over 2/3 of the bits. 

Additionally, we generalize the Hamming distance approximation of a solution structure to approximating it with respect to the edit distance metric,  obtaining similar lower bounds.

\end{abstract} 


\section{Introduction}

The phrase alignment problem arises in the context of machine translation and natural language inference \cite{maccartney2008phrase}.  It is a common task in these areas to determine  whether one sentence can be converted into another by replacing blocks of text with semantically equivalent blocks, and possibly changing the order of the blocks.  For example, the sentence ``The president of the USA spoke on New Year's day'' and the sentence ``On January 1st, Obama gave a talk''  convey the same information;  we can convert the former into the latter by replacing ``the president of the USA'' with ``Obama'', ``on New Year's day'' with ``on January 1st'' and ``spoke'' with ``gave a talk''. 

Following the setting of DeNero and Klein \cite{phrase-align}, we call the two sequences of words (tokens) to be aligned ``sentences'', a consecutive block of words  a ``phrase'', and an aligned pair a ``link''.  A set of links such that each word (in either sentence) occurs in exactly one link is called an alignment of the sentences.  In the example above, an alignment can be \{(the president of the USA, Obama), (spoke, gave a talk), (on New Year's day, on January 1st)\}.  In practice, there can be various degrees of how good a certain link is:  there is a better correspondence between ``Obama'' and ``the president of the USA'', than between ``Obama'' and ``the president'', for example; ``spoke'' and ``gave a talk''  might not be as close semantically as the other two links. But either of them would be better than aligning ``USA'' with  ``Year's day''. Thus, another parameter of the problem is a scoring function assigning a weight to each potential link. The weighted sentence alignment problem is defined then as finding a phrase alignment with the best weight. In the  machine translation application,   where each phrase is linked with its potential translation,  statistical models are used to estimate the weight of each link as its probability and the weight of an alignment is the the product of weights of its links.  

 In a more general statement of the problem, in particular in the natural language inference setting \cite{maccartney2008phrase}, the original sentence (text) can contain much more information than the resulting sentence.  However, it can be reduced to the bijective case by padding the target sentence with null words (half the number of words of the original sentence suffices), and setting the weight of links between any phrase over the null words and any phrase of the original sentence to be 1, and weight of any link with a phrase involving both null and non-null words to be 0.   

 In \cite{phrase-align},  DeNero and Klein show that the weighted sentence alignment problem is NP-hard, with its decision version being NP-complete.  Several approaches are commonly used to deal with NP-hardness in practice: restricting the problem,  heuristics and approximation algorithms.  An early example of such a restriction is a bag-of-words alignment of  IBM models  1 and 2 for statistical machine translation \cite{brown1993mathematics}.  In this setting,  there is no need to determine a partition of the source and target sentences into phrases of the optimal alignment, which significantly reduces computational complexity of a problem.  We will focus on the general alignment of phrases to phrases, as well as the setting of the IBM models 3, 4 and 5, in which phrases in one string are matched to the words in the other: this variant of the problem is already NP-complete (unless the alignment has to respect the order of phrases).  To simplify the problem, we will assume, following \cite{phrase-align},  that the probability (that is, weight) of each link is given as part of the input.
 
 Heuristics have been a popular approach for phrase alignment, used both as a direct application of a heuristic and in the context of modelling a problem in a Integer Linear Programming framework, and then invoking  heuristics-based solvers for ILP.  In particular,  hill climbing has been used in \cite{marcu2002phrase,och2003systematic,birch2006constraining} and simulated annealing in \cite{maccartney2008phrase} to solve the problem of partitioning strings into phrases.   However, although useful in practice, such heuristic algorithms give no guarantee of the closeness to optimality.  
  
In this paper we will focus on the complexity of approximating  an optimal alignment. However, we will consider a somewhat different notion of an approximation.  Usually, an approximation algorithm produces a solution with a value close enough to the value of an optimal solution (for example, an alignment with probability at least half that of the optimal). But such an alignment can be very different from an optimal alignment. This invites a natural question: is it possible to compute a solution, an alignment, which is guaranteed to share a significant fraction of links with an optimal solution? For example, is it possible to compute a translation in which most of the source sentence is translated correctly, even if the incorrectly translated part may bring the overall probability of the alignment down to 0? To investigate this type of approximation, we will use the structure approximation framework of \cite{HMRW07}.
 
 \subsection{Approximating solution structure}
 
Motivated by cognitive psychology applications such as  the Coherence problem,  Hamilton, M\"{u}ller, van Rooij and Wareham \cite{HMRW07} presented  variant of  approximation which they called a \emph{structure approximation}. This framework extends the notion of finding solutions  close in value to the optimal  to close according to a specified metric.  More precisely, the description of a problem includes a  distance function $d(y,z)$ which may depend on the input, and an approximate solution  $y$ is considered good if  $d(y,z)$ is sufficiently small for some optimal solution $z$.  This generalizes the standard notion of approximation as the distance $d(y,z)$ can be defined as a log of the ratio of values of solutions $y$ and $z$.  In a follow-up paper \cite{Todd-journal}, this approach was applied to other problems such as the coherence model of belief fixation in cognitive science. 

\cite{HMRW07} present a number of lower bounds results for  arbitrary distance functions such as  showing  that there are no NP-hard problems with a structure analogue of FPTAS for an arbitrary function. Among the other distance function they consider, the most prominent is the  Hamming distance. This is a very natural metric for comparing how close two solutions encoded as binary strings are. For example, in the Hamming approximation for Max3SAT a solution close to the optimal would be considered a solution which differs from an optimal in  few variable assignments, even if these variable assignments dramatically decrease the number of satisfied clauses.  

Several other papers include results that can be interpreted as lower bounds for structure approximability with respect to Hamming distance. The reconstruction of a partially specified NP witness, considered in the 1999 paper by Gal, Halevi, Lipton and Petrank \cite{GHLP99}, is probably the first result along these lines. There, they show that it is possible to reconstruct a satisfying assignment to a formula from $N^{1/2+\epsilon}$ bits of a satisfying assignment of  a related (though larger) formula. Their proofs  rely on erasure codes, thus $\epsilon$ is a fixed parameter.  They also consider Graph Isomorphism, Shortest Lattice Vector and Clique/Vertex Cover/Independent set.     In 1999, Kumar and Sivakumar  \cite{KS99} showed that for any NP problem there is a verifier with respect to which all solutions are Hamming-far from each other: make the witnesses to be encodings of natural witnesses to the original problem by some error-correcting code, the verifier decodes the witness and then checks it using the original verifier. Then, list-decoding allows one to find a correct codeword for the witness from a string which is within $n/2+n^{4/5+\gamma}$ Hamming distance from it.   Following this, Feige, Langberg and Nissim \cite{FLN00} show that some natural verifiers (e.g., binary strings directly encoding satisfying assignments for variants of SAT, encoding sequences of vertices for Clique/Vertex Cover, etc) are often hard to approximate to within Hamming distance $n/2-n^\epsilon$ for some $\epsilon$ dependent on the underlying error-correcting code.  Guruswami and Rudra \cite{GR08} improve this $\epsilon$ to $2/3+\gamma$,  but on the negative side argue that methods based on error-correcting codes can only give bounds up to $n/2-O(\sqrt{n \log n})$.   

The recent paper of Sheldon and Young \cite{SY13} settles much of the Hamming distance approximation question, providing the lower bounds of $n/2-n^\epsilon$ for any $\epsilon$ for many of the problems considered in \cite{FLN00}, as well as upper bounds of $n/2$ for several natural problems including Weighted Vertex Cover, and  a surprising $n/2+O(\sqrt{n \log n})$ lower bound for the universal NP-complete language. The latter result they extend to existence of such very hard to approximate verifiers for all paddable (in Berman-Hartmanis \cite{BH77} sense) NP languages, improving on \cite{KS99}. Their proof techniques avoid error-correcting codes altogether, instead combining amplification with search-to-decision (Turing)  reductions and downward self-reducibility.  

\subsection{Our results} 

In this paper, we analyse the complexity of approximating solution structure of the weighted sentence alignment problem (WSA), in particular its variant in which phrases in the source sentence are aligned with words in the target sentence (PWSA problem). We show that for PWSA,  even when the weight function is restricted to take $\{0,1\}$ values, computing an alignment which agrees with an optimal on at least $n/2+n^\epsilon$, for any constant $\epsilon >0$,  links is NP-hard, where $n$ is the length of the source sentence.  Moreover, the hardness stems from the problem of the partitioning the source sentence into phrases: we show how to modify the NP-hardness proof in such a way that the optimal alignment can be recovered directly from such partition.  More specifically, we define a compact solution representation for that problem to be a binary string encoding the locations of  phrase boundaries, and show that computing a string which agrees with it on at least $n/2+n^\epsilon$  positions (that is, a string within Hamming distance $n/2-n^\epsilon$)  is already NP-hard.  Note that since expected Hamming distance between any string with $n/2$ 1s and a random string with $n/2$ 1s is $n/2$, there is a randomized algorithm giving an expected Hamming approximation  $n/2$. Therefore, our results are tight. 

 For the more general case where the target string is required to be partitioned into phrases as well (and thus the solution represents partitions for both strings), we obtain a weaker bound requiring a $2n/3+n^\epsilon$ agreement for NP-hardness. 

A different metric of the distance between two solutions encoded in this form is an edit distance:  there, a string resulting from shifting a consecutive group of phrases  by one word is considered to be distance 2 from the original, even if the shift has affected a significant portion of the string.   We show how the Hamming distance approximation results can be extended  to give edit distance approximation for two standard NP-hard problem 3SAT and VertexCover, and how to apply this technique to give lower bounds on edit distance approximation of the WSA and PWSA problems.  To our knowledge, these are the first, if mathematically simple,  such lower bounds on approximating solution structure with respect to edit distance (although \cite{HMRW07} do give a lower bound on edit distance solution structure approximation for the Longest Common Subsequence problem in the parameterized setting).

\section{Preliminaries}

Following DeNero and Klein \cite{phrase-align}, we formally define a \emph{weighted sentence alignment (WSA)} problem as follows. Let $e$ and $f$  be sentences.  The phrases in  $e$ are represented by a set $\{e_{ij}\}$, where $e_{ij}$ is a sequence of words from in-between-word position $i$ to $j$ in $e$; $f$ is represented by $\{f_{kl}\}$ in the same fashion.  A link is an aligned pair of phrases $(e_{ij},f_{kl})$.  An alignment is a set of links such that every word (token), in either sentence,  occurs in exactly one link (here, we treat each occurrence of a word as a separate word).   A weight function  $\phi: \{ (e_{ij},f_{kl}) \} \to \mathbb{R}$ assigns a weight to each link. A total weight of an alignment $a$, denoted $\phi(a)$,  is a product of weights of its links.  Now,  an optimization version of the weighted sentence alignment problem asks, given $(e,f,\phi)$,  to find the alignment with the maximum weight. A decision version of this problem can be stated as finding an alignment $a$ of weight $\phi(a) \geq 1$.   

\begin{theorem}\cite{phrase-align}\label{sat-wsa} 
The decision version of the WSA problem  is NP-complete. 
\end{theorem} 
\begin{proof} 
 DeNero and Klein in \cite{phrase-align} show NP-hardness of WSA by the following  reduction from 3SAT. Let $F$ be a formula with $n$ variables and $m$ clauses. The construction will produce an instance  $I$ of WSA consisting of sentences $e$ and $f$, and a function $\phi$ such that there is an alignment of weight (at least) 1 in $I$ if and only if  $F$ is satisfiable.  For that, let sentence $e$ consist of blocks of words as follows, with one word for each occurrence of a literal:   $x_i^1\dots x_1^{p_i} \bar{x}_i^1 \dots \bar{x}_i^{q_i}$, where $p_i$ and $q_i$ are the number of positive and negative occurrences of $x_i$ in $F$, respectively. Thus, the length of $e$ will be $ \leq 3m$, with equality if every clause in $F$ contains exactly 3 literals.  Now, the sentence $f$ will contain two types of words. The first $m$ words, $c_1 \dots c_m$, will correspond to the clauses of $F$. They will be followed by ``slack words'' $s_1 \dots s_n$, one for each variable in $F$.  Finally, the function $\phi$ will only have values 0 and 1, and it will have the value 1 in two cases. First,  if the link is of the form $(c_i, l_k)$, where literal $l_k$ occurs positively in clause $c_i$ (for all occurrences of $l_k$).  This will be used to align each clause with a literal that makes it true. Second, each slack variable $s_i$ corresponding to a variable $i$ will be aligned with all possible substrings of $x_i^1\dots x_1^{p_i} \bar{x}_i^1 \dots \bar{x}_i^{q_i}$ in which either all positive or all negative copies of the variable (or both)  are present. For example, if there is one positive occurrence of $x_i$ and two negative occurrences of $x_i$, then the links with $\phi(e_{i,j}, f_{k,l})=1$ have $f_{k,l}=s_i$ and $e_{i,j}$ either $x_i \bar{x}_i \bar{x}_i$, or $\bar{x}_i\bar{x_i}$, or $x_i\bar{x}_i$, or $x_i$. The first one covers both positive and negative, the second covers all negative, and the last two all positive occurrences of the literal.  These slack variables are needed to ensure that either only positive or only negative literals are left unmatched to be aligned with clause words.  
 
To see that this reduction works, note that a satisfying assignment becomes an alignment in which every clause word is matched with one literal that makes it true (starting from the front of the block for positive and end of the block for negative), and slack variables cover the literals that remain unmatched to clauses.  For the other direction, note that there is exactly one link for each slack variable:  if it is matched with a block that contains all positive occurrences of the corresponding variable in $F$,  the corresponding variable can be set to false, 
otherwise it can be set to true (if it is matched with the block containing all occurrences, then either assignment works).  
 
Assuming that $F$ has exactly 3 variables per clause, $|e|=3m$, $|f|=m+n$, and $|\phi|  \leq (3m)^2(m+n)^2$, therefore the resulting instance is polynomial size, and the reduction runs in polynomial time. 

Therefore, WSA is NP-hard.  As an  alignment can be checked for validity (by asserting that each word appears exactly once) and the  weight of the alignment can be computed in polynomial time, the decision version of WSA is NP-complete.  
  \end{proof}

Alternatively, NP-hardness of WSA can be shown by a reduction from the VertexCover problem. There, we are given an undirected  graph  $G=(V,E)$ with $n$ vertices and $m$ edges, and asked whether there exists a subset of $k$ vertices called a cover such that every edge has as its endpoint  at least one vertex in the cover.  In an optimization version, a minimal-size such cover is sought.  To show $VertexCover \leq_p WSA$, construct the instance as follows. The words of $e$ will be blocks of copies of each vertex $v_i$, where the length of each such block is the degree of $v_i$, denoted $deg(v_i)$, plus 1, so $|e|=2m+n$.  The words of $f$ will be of three types. The first $m$ words $c_1 \dots c_m$ will correspond to edges of $G$;  the next $n$ words are the ``slack variables'' $s_1 \dots s_n$ covering leftover copies of vertices, with one extra copy always covered by $s_i$,  and the final  $n-k$ words $t_1 \dots t_{n-k}$  in $f$ will ensure that the size of the cover is at most $k$. Thus, $|f|=m+n+(n-k)=m+2n-k$.  With this intuition, define $\phi$ so that  $\phi(v_{i,j},c_l)=1$ if edge $c_l$  has $v_i$ as its endpoint  (for each copy $v_{i,j}$ of $v_i$),  then $\phi(v_{i,j} \dots v_{i,\deg(v_i)+1}, s_i)=1$ for each $i$ and all  $j$, $1 \leq j \leq deg(v_i)$. Finally, each $t_l$  can cover the full block for every vertex (except for the last copy), so $\phi(v_{i,1} \dots v_{i,deg(v_i)}, t_l)=1$ for every $t_l$ and every $v_i$. 

If there is a vertex cover of size $k$ in $G$, then an alignment in the constructed instance will  link all vertices other than the $k$ vertices in the cover with $t$-variables, will link each edge with a copy of a vertex in the cover (in order starting from $v_{i,1}$), and variables $s_i$ will be linked with a block of remaining copies of the corresponding vertices (consisting of at least one special copy, more if some edges have both endpoints in the cover).  For the other direction, variables $t_l$ denote vertices not in the cover, so the cover consists of the remaining vertices.  If there is a cover of size smaller than $k$, then some $s_i$ variables  align with the whole block corresponding to such extra $v_i$, which is allowed by our definition of $\phi$. 

\subsection{Defining a natural witness for WSA}

Before we can talk about structure approximation of WSA, we need to define what is meant by the witness (or feasible solution) to the  WSA problem. Here, we will consider an alignment of any weight to be a feasible solution; the question remains how to represent an alignment.  In DeNero and Klein \cite{phrase-align}, an alignment is visualized as a matrix with words of $e$ as columns, words of $f$ as rows and a cell $(i,k)$ highlighted  (say, set to 1) if the  block with the $i^{th}$ word of $e$ is linked to the block with the $k^{th}$ word of $f$. Each link thus becomes a rectangular all-ones block in the matrix.  This representation is not the most efficient in terms of space, although it is convenient for visualization of the solution.  In particular, for the instances coming from the $3SAT  \leq_p WSA$ reduction above, any feasible solution will only have $3m$ cells out of $3m \times (m+n)=N$ possible cells highlighted.   In this case, it is trivial to approximate the witness to an instance of WSA produced from this 3SAT reduction: an all-zero matrix already gives a $N-(m+n)$ Hamming distance approximation.

Now, notice that the reduction above proves NP-hardness for a special case of the problem: that where all phrases in $f$ are single words. For this restricted problem, a Hamming distance (and therefore an edit distance) approximation by an all-zero matrix is  $|e|*|f|-|f|$ close to any solution. One may object that an all-zero matrix is not a valid alignment: here, 
we can construct an alignment by matching first $|f|-1$ words of $e$ with words of $f$, and  all the remaining words of $e$ as one phrase to the last word of $f$. This gives us a  $|e|*|f| - 2|f|$ Hamming approximation for the alignment represented as $|e| \times |f|$ matrix. 

As we are looking for natural (and compact) witnesses, we will use a different representation of the solution. For that, notice that finding a solution to WSA involves solving two problems: first, we need to determine how to break each sentence into phrases, and second, to determine an optimal alignment using only links involving these phrases.  So a feasible solution can consist of two components: the first component with two binary strings of length $|e|-1$ and $|f|-1$, with $1$ in between-phrase positions and $0$ otherwise.  The second component can list the order of phrases in $f$ mapping to phrases in $e$; if there are $n$ phrases in each, then the length of that component is $n \log n$.  

What part of computing this witness, and thus of solving the WSA problem, is the hardest? Consider again the set of instances of WSA resulting from the reduction. We would like to define a special case of WSA for which we could use as small a witness as  possible, and still have the NP-hardness reduction above work.  As noted above, one special property of this reduction is that it always produces a partition of $f$ where every phrase is exactly one word. The information encoded in the second part of the witness described in the previous paragraph, the string of $|f|-1$ bits denoting the phrase boundaries in $f$, is therefore redundant.  

Secondly, $\phi$ involved in the reduction has a special property that it can only take values  $0$ and $1$. In that case,  after solving the first part of the problem (finding splitting points between phrases in $e$ and $f$), the second part can be computed in polynomial time by the standard network flow algorithm for bipartite perfect matching, with  phrases of $e$ and $f$ forming the vertices of the bipartite graph, and an edge connecting two vertices $v$ and $u$ iff $\phi(v,u)=1$.  Thus, in this case it is enough to compute a  witness which contains only the binary strings denoting splitting points between phrases, as described above.  

Now, combining the two restrictions we will define a problem PWSA, which is a special case of WSA  satisfying the properties above. 

\begin{definition}[PWSA] 
The PWSA (for ``partition'' WSA)  problem is defined as follows. Given as input $(e,f,\phi)$ where $\phi: \{(e_{ij}, f_{kl})\} \to \{0,1\}$, find a partition of $e$ into phrases such that there is an alignment of weight 1 of phrases in this partition with words of $f$.  

The  natural witness  $w$ for PWSA will be a binary string $w_1 \dots w_{|e|-1}$ such that  if $e_{ij}$ is a phrase in the optimal alignment, then $w_i=w_j=1$, or $w_j=1$ and $i=0$, or $w_i=1$ and $j=|e|$; and $\forall k, i <k<j, w_k=0$.   Note that $w$ has to have  $|f|-1$ 1s for any valid alignment.  
\end{definition} 

Here, the NP-hardness follows by the same $3SAT \leq_p WSA$ reduction as in theorem \ref{sat-wsa}, where the satisfying assignment is recovered from $w$ by running the network flow algorithm and determining, as before, the values of the variables of $F$ from the links with slack variables $s_i$.  Moreover,  for variables with more than two positive and two negative occurrences the value can be determined directly from $w$.  Suppose a slack variable covers all positive occurrences of a variable $v$, and leaves out some negative occurrences. Then, there will be no splitting points within the block denoting the positive literals, but there will be as many splitting points for the negative literals as there are clauses which use them. From that, already, it can be inferred that the negative occurrences were used to satisfy the clauses, thus the variable needs to be set to false. 
So if a substring  $w_{ij}$ of $w$ corresponding to a block of encoding a literal $v$ (without the endpoints)  is of the form $1111....0000$, then we can immediately infer that  $v=true$, otherwise if it is of the form $000....1111$, $v=false$.  It would not work if there is exactly one  positive or negative occurrence of a variable;  but this can be resolved by modifying   the reduction  so that there is always an extra ``$v_i \bar{v_i}$''  (or a single dummy variable) in the middle of each block, and $\phi(x\dots x)=\phi(\bar{x} \dots \bar{x})=0$. Then, the partition of $e$ uniquely specifies the optimal alignment.

\section{Edit distance inapproximability}

Consider $d_E(y,z)$ to be the \emph{edit distance} between strings $y$ and $z$, that is, the number of insert, replace and delete a symbol operations needed to convert $y$ into $z$.   This function, even though in some respect related to Hamming distance, nevertheless has a very different behaviour. For example, a string $01010101$ and a string $10101010$ have the maximal Hamming distance of $n=8$, however their edit distance is just $2$, corresponding to deleting a $0$ in front and inserting it in the back of the string. For Hamming distance, a random string is expected to be within $n/2$ from any string, but  it is not clear what expected edit distance between two random strings is.  If two strings are far in the edit distance though, then in particular they are far in the Hamming distance. So lower bounds on edit distance approximability imply lower bounds for the Hamming distance, but the reverse is not immediate. 

However, in case when one of the strings is a string of all 0s or all 1s then the two notions coincide, as long as the length of the approximating string is the same. Indeed, even edit distance with transpositions to a string of all 1s from any given string is equivalent to  Hamming distance. 

\begin{lemma} 
For any string $x$ of length $n$, its Hamming distance to a string of $n$ 1s is equal to the edit distance. 
\end{lemma} 

The proof follows directly from the fact that only replacements and insertions introduce 0s, and each insertion needs to have a corresponding deletion. Now, Sheldon-Young \cite{SY13} proof that a natural witness for SAT cannot be Hamming-distance-approximated to within $n/2-n^\epsilon$,  for any constant $\epsilon>0$, proceeds as follows.  First, note that it is enough to  have an algorithm determining the value of one variable; the formula is then simplified and the process is repeated until the whole assignment is revealed. Now, the proof proceeds by amplifying an arbitrary variable $z_i$   $n^{1/\epsilon}$ times, that is introducing $n^{1/\epsilon}$ new variables and adding clauses stating that they are equivalent  to $z_i$. Now, if there is a polynomial-time algorithm that is guaranteed to return a witness within $n/2-n^\epsilon$ Hamming distance of a satisfying assignment, then such a string will be correct on majority of copies of $z_i$. Taking the majority thus gives the correct value of this variable, and repeating the process $n$ times, substituting computed values on each iteration, results in a  satisfying assignment.  The resulting algorithm for SAT  will run in time $n^{O(1/\epsilon)}$ times the running time of the assumed polynomial-time approximation algorithm, which is polynomial when $\epsilon$ is constant. 

\begin{theorem} 
If there is a polynomial-time algorithm that, for some constant $\epsilon >0$,  can approximate the natural witness to SAT to within edit distance $n/2-n^\epsilon$, then P=NP.
\end{theorem} 
\begin{proof}

Note that a natural witness for  this problem consists of either $n^{1/\epsilon}$ 0s or $n^{1/\epsilon}$ ones, together with  $n-1$ symbols of arbitrary values for the rest of the variables; moreover, we can assume that all values of the copies of $z_i$ are together, for example forming the first $n^{1/\epsilon}$ positions of the string.   Now, suppose there is an algorithm that approximates the satisfying assignment above,  with $n^{1/\epsilon}$ copies of $z_i$, to within edit distance $N/2-N^\epsilon$ rather than Hamming distance, where $N=n+n^{1/\epsilon}$.  Let $y'$ be a string returned by the approximation algorithm and $y$ the corresponding optimal solution. Consider only the first $n^{1/\epsilon}$ positions in $y'$, ones corresponding to the copies of $z_i$.  Without loss of generality, assume that $z_i=1$ in $y$.  These positions can be changed to 0 (to obtain $y'$) by either a replacement or an insertion/deletion pair moving values of the remaining $n-1$ variables into the first $n^{1/\epsilon}$ positions. But as discussed above, in this case the number of insert/delete pairs is at least as large as the number of replacements. Therefore, the same argument as for the Hamming distance applies, and bounding the edit distance between $y$ and $y'$  by  $N-N^\epsilon$ means that majority of the copies of $z_i$ in $y'$ have a correct value. 
 Note also that this argument works even if transposition operations are allowed. 
\end{proof} 

A similar argument can be used to show  $n/2-n^\epsilon$ lower bound for the edit distance approximation of VertexCover; however, as it will involve a string of 1s and a string of 0s, the only edit distance operations allowed will be insertions, deletions and replacements. Recall that in the MinVertexCover the goal is to determine a minimal set of vertices such that every edge has at least one endpoint in the cover; the decision version  VertexCover asks to determine if there is a cover of size at most $k$.   A natural witness to VertexCover is  a binary string of length $n=|V|$, where a bit corresponding to a vertex is 1 iff that vertex is in the cover.  In the \cite{SY13} proof of Hamming distance inapproximability of this problem,  in an input graph a copy of an arbitrary vertex $v$ is made and an even-length path on $\geq 2n^{1/\epsilon}$ vertices  is added between $v$ and its copy $v'$. 
Now, as a (minimal) vertex cover of an even-length path consists of either all even or all odd vertices, we say that the original $v$ is in the $k+n^{1/\epsilon}$ cover if all even vertices are in that cover, otherwise $v$ is not in the cover.  Then the argument proceeds by showing that the majority of the vertices on the path will be correctly placed by the same calculation as for SAT above.

\begin{theorem}\label{vc-edit}
Unless P=NP, no polynomial-time algorithm can approximate the natural witness to VertexCover within edit distance  $n/2-n^\epsilon$, for any constant $\epsilon >0$. 
\end{theorem}
\begin{proof} 
Consider the \cite{SY13} construction described above, but with  a different naming convention for the variables in the witness.  Let variables $v_1 \dots v_n$ be the original variables, $v'$ a copy of a selected variable e.g. of $v_1$,  $u_1 \dots u_{n^{1/\epsilon}}$ be even variables on the path from $v$ to $v'$ and $w_1 \dots w_{n^{1/\epsilon}}$ be the odd variables on that pass. Now, in the witness the first $n^{1/\epsilon}$ positions will correspond to the $u_i$ variables, followed by $v_i$s, in turn followed by the $w_i$s. 

Now, the same kind of argument as before applies. The witness, a characteristic string of a vertex cover of size $K=k+n^{1/\epsilon}$,  will be encoded by  either a string of $n^{1/\epsilon}$ 0s followed by some string of length $n+1$ followed by $n^{1/\epsilon}$ 1s, or a similar string with 0s at the beginning and 1s at the end. Now,  similarly to the SAT  construction, we would like to argue that a sequence of $N/2-N^{\epsilon}$ of arbitrary edit operations (insertions, deletions, replacements) would not result in any string that differs from the original on the $u$-part and $w$-part in more than  $N/2-N^{\epsilon}$ positions. 

Consider a pair of insert/delete operations applied to the above string  encoding a $K$-cover. Suppose, without loss of generality, that the correct string starts with 1s and  ends with 0s.  Consider deleting a value from the $u$ part of the string and inserting it into the $w$ part. Now, the middle part of the string, corresponding to the $v$ variables, could become maximally far from the encoding of the $K$- vertex cover at that point (i.e., if it was of the form 01010101), however to determine whether $v$ is in the cover, only variables $u_i$'s and $w_j$'s are relevant.  A pair of insert-delete operations then introduces at most one 0 into the $u$  part (by shifting the $v$ part into it), and at most one 1 into the $w$ part by insertion. Therefore, the ``damage done'' to these parts of the string is no more than from doing two replacements, and the argument still applies to an already corrupted string. 

Therefore,  if there exists a structure approximation  algorithm for vertex cover that can consistently return a string within edit distance $n/2-n^\epsilon$ from an optimal cover, then this algorithm can be used to determine exactly whether any given variable is in the intended cover. By Turing/search-to-decision reduction, from there the actual cover can be computed. In this reduction, if a vertex was determined to be in the cover, then recurse on a graph without this vertex, and otherwise recurse on a graph without this vertex and all of its neighbours.  
\end{proof}

So far, we have discussed the complexity of approximating an NP witness, however in majority of practical problems it is approximating an optimal solution which is of interest.  But since lower bounds on decision problems imply lower bounds on optimization problems, the results above give inapproximability of the optimization version of this problem, in particular MaxSAT and MinVertexCover.

\section{Hamming distance and edit distance  inapproximability of PWSA and WSA} 

In this section we will show that  PWSA cannot be Hamming or edit distance structure approximated to within $n/2-n^\epsilon$, with respect to the witness defined above. From this, the structure inapproximability of WSA can be derived, albeit with weaker parameters. Note that a random string with $n/2$ 1s  has expected Hamming distance $n/2$ from any given string with $n/2$ 1s; the larger disparity between the number of 0s and 1s gives a better expected Hamming distance. Thus, there is a randomized algorithm approximating PWSA to within Hamming distance $n/2$, but the results below show that doing better than that by a small inverse polynomial amount is NP-hard. 

\begin{theorem}[Hamming inapproximability of PWSA] \label{pwsa-hamming}
Let $(e,f,\phi)$ be a valid input to PWSA. If there is a polynomial-time algorithm $A(e,f,\phi)$ computing a string $w$ which is within Hamming distance $n/2-n^\epsilon$ of a witness for any constant $\epsilon >0$, then P=NP. 
\end{theorem} 

\begin{proof} 
We will show how to use such a structure approximation algorithm $A$ for PWSA to compute the exact value of the first variable in $F$, in a manner similar to the proof of Hamming inapproximability of SAT. 

Let $F$ be a formula on $n$ variables and $m$  clauses.  Choose $k$ such that $n^k >1.5m$.  Now, augment  $F$  with $n^{k/\epsilon}$ copies of the dummy clause $(v \vee \bar{v})$ to obtain a new formula $F'$.  If the reduction from theorem \ref{sat-wsa} is applied to this $F'$, it will have an effect of  introducing  $n^{k/\epsilon}$ copies of the literal $v$ and $n^{k/\epsilon}$ copies of the literal $\bar{v}$ as additional words of $e$ (that is, the first $n^{k/\epsilon} + p$ words of $e$ will be copies of $v$, and the following $n^{k/\epsilon}+q$ words of $e$ will be copies of $\bar{v}$, where $p$ and $q$ are the numbers of positive and negative occurrences of $v$ in the original $F$.)  The clauses $(v \vee \bar{v})$ will become $n^{k/\epsilon}$ new words in $f$ (say first $n^{k/\epsilon}$ words of $f$).  Finally, $\phi(e_{ij}, f_{kl})$ is defined as before with respect to the augmented formula.   This amplification preserves the correctness of the reduction, as  the link $(e_{ij},s_1)$ forces only copies of $v$ or only copies of $\bar{v}$  to be used to satisfy the dummy clauses.  Now,  if $w$ is a correct witness (of length $N=3m+2n^{k/\epsilon}-1$) to this instance, the value of $v$ can be determined immediately: if $w$ starts with a string of at least $n^{k/\epsilon}$ 1s, then $v=true$, and if $w$ starts with at least $n^{k/\epsilon}$ 0s, then $v=false$. 

Suppose that there is an algorithm $A$ that returns a ``corrupted'' string $w'$ which agrees with $w$ on at least $N/2+N^\epsilon$ bits. Here, we are not even concerned whether $w'$ is a valid alignment (i.e., has $|f|-1$ ones); any such $w'$ will work.   That is,  $w'$ agrees with $w$ on 
\( (3m+2n^{k/\epsilon}-1)/2 + (3m+2n^{k/\epsilon}-1)^\epsilon\) \( \geq (3m +2n^{k/\epsilon}-1)/2 + n^k\) positions. Now, suppose that all the errors lie within the $2n^{k/\epsilon}$ positions corresponding to extra copies of $v$ and $\bar{v}$.   Since we chose $k$ such that $n^k>1.5m$,  and ignoring $-1/2$, 
there are at least  $n^{k/\epsilon}+n^k-1.5m > n^{k/\epsilon}$ correct bits in that block, that is more than half of copies of $v$ and $\bar{v}$ are computed correctly. Taking majority now gives us the correct value of $v$. 
   \end{proof} 

This  result can be extended to show edit distance inapproximability of PWSA using the ideas from the edit distance inapproximability proof for  VertexCover. 

\begin{corollary} 
PWSA cannot be approximated in polynomial time to within edit distance  $n/2-n^\epsilon$ for any constant $\epsilon >0$ unless $P=NP$.  
\end{corollary} 
\begin{proof} 
We will use the same class of  instances as in theorem \ref{pwsa-hamming}. Note that the substring of $w$ that we are interested in is $w_1 \dots w_r$, where $r=2n^{k/\epsilon}+p+q$, which is the block corresponding to the first variable $v$ in $F$.  In a correct witness, this substring is either of the form $1111....000000$ or $000....11111$, with the number of 0s and 1s at least $n^{k/\epsilon}$ each.  Now, suppose an approximation algorithm $A$ produces a string $w'$ which is edit distance $N/2-N^\epsilon$ of $w$; that is, $w'$ can be converted to $w$ with at most $N/2 + N^\epsilon$ insertion, deletion and replacement operations.  Consider a substring $w'_1\dots w'_r$  in $w'$.  As for the case of VertexCover,  we can argue that the Hamming distance between $w_1 \dots w_r$ and $w'_1  \dots w'_r$ is at most $N/2-N^\epsilon$.  Indeed, suppose for the sake of contradiction that the Hamming distance between $w_1 \dots w_r$ and $w'_1 \dots w'_r$ is greater than the edit distance between these two substrings. As they have the same size, the number of insertions is the same as the number of deletions.  Now, it is sufficient to say that the pair insertion/deletion can introduce at most one 0  in the ``1111...1'' part, and at most one 1 in the ``0000..000'',  by the same argument as in theorem \ref{vc-edit}.  Therefore, the Hamming distance inapproximability implies edit distance inapproximability with the same parameters. 
\end{proof} 

In the proofs above, we have shown inapproximability results for the problem PWSA, in which the second sentence is assumed to be partitioned  as one word per phrase. A more realistic scenario would be to assume that the witness consists of the partition strings for both $e$ and $f$ (here, we are still assuming that $\phi$ takes values in $\{0,1\}$). The corollary below shows that for a weaker bound, there is still an inapproximability. The weakening here comes from the fact that our block becomes a smaller fraction of the total length of the witness, since  $f$ contains $n^{k/\epsilon}$  words corresponding to the dummy clauses. 
  
\begin{corollary} 
WSA with $\phi \in \{0,1\}$ cannot be approximated to within Hamming distance or edit distance $2n/3+n^\epsilon$ for any constant $\epsilon >0$. 
\end{corollary} 
\begin{proof} 
Consider the same reduction as before, but now the witness is of length $|e|+|f|$ and encodes partition into phrases of $f$ as well as of $e$.  Thus, the total length $N$ of the witness becomes, ignoring ``-1''s,  $N=(3m+2n^{k/\epsilon})+(n^{k/\epsilon}+m+n)$ $=4m+3n^{k/\epsilon}+n$.   If the calculation above is done with this value of $N$, then we end up with  only  $0.5 n^{k/\epsilon}$ guaranteed correct positions in our $2n^{k/\epsilon}$ block of interest. 
We need $c$, $0<c<1$, such that $N*c + N^\epsilon - (N-2n^{k/\epsilon}) > n^{k/\epsilon}$; choosing $c=2/3$ satisfies this condition. 

\end{proof} 

\section{Conclusions} 

In this paper we have considered the problem of approximating solution structure for the weighted sentence alignment problem and its phrase-to-word variant.  We have shown that  a partition of a source string into phrases for which there is an optimal alignment  is hard to approximate to within Hamming distance or edit distance $n/2+n^\epsilon$ for all $\epsilon$, where $n$ is the length of the source string. We adapted the framework of \cite{HMRW07} and the techniques of \cite{SY13} for this task, in particular showing how the Hamming distance results of \cite{SY13} can be extended to edit distance for several problems. 

Additionally, the discussion of the most compact representation of the solutions to WSA and its variants suggests a direction for the parameterized complexity analysis of this problem. The ``source of intractability'' there seems to be the partitioning task.  It is known, for example, that limiting the distance, in terms of position,  at which the linked phrases can be (generalizing the ``monotone WSA'', where the alignment must preserve the order of phrases) allows the problem to be solved in polynomial time by a dynamic programming algorithm \cite{denero-thesis}.  Can limiting the number of phrases or the length of phrases give a fixed-parameter tractable algorithm for WSA or would it be W[1]-hard?  Note that limiting both the number and the length of phrases does give an FPT algorithm,  but it is not interesting since bounding both puts a limit on the length of the string itself.  Another note is that the reduction from Vertex Cover contains a block of $k'=n-k$  $t$-words; thus, considering it a reduction from $k'$-independent set, the parameter $k'$ suggests W[1]-hardness. However, this does not give a natural parameter of WSA corresponding to $k'$, as the length of $f$ depends on the size of the graph. Yet another parameter that can be considered,  in the $\{0,1\}$ framework, would be the maximal  number of links  of weight 1 per phrase. As real-world sentences to be translated tend to be of restricted types, such parameterized analysis  may explain the success of heuristics and integer linear programming approach to solving WSA. 

The analysis of the approximation algorithms based on the integer linear programming formulation of the WSA used by \cite{phrase-align} and others is another interesting question.  Is there a linear programming-based or SDP approximation algorithm for WSA?  And would an approximation produced by such algorithm agree with the elements of the optimal solution enough to give a matching upper bound for the approximating solution structure (as it is for weighted MinVertexCover \cite{SY13})?   Here we did not go into details of the underlying statistical models, rather working in the simplified bijective setting of \cite{phrase-align}.  How would such upper bounds apply in a more general context of phrase alignment problems,  both with respect to optimality conditions and the requirement that alignment has to be bijective?

 Finally, in this paper we considered the weighted sentence alignment problem and distance functions Hamming distance and edit distance. Exploring the setting of structure approximation further,  it would be interesting to see if there is a generic way to build a lattice of hardness implications for various metrics. We conjecture,  in particular, that any metric with a certain ``locality property'' (that is, one ``unit of change'' only affects a small, though not necessarily constant number of positions)  should be inapproximable by generalizing Hamming distance results.  Alternatively,  one wonders if there is a non-trivial, practically interesting metric for which there is, indeed, a fast approximation algorithm for any NP-hard problem.  In that respect, considering various metrics and their interrelation with respect to  computational problems is a promising area with a possibility for new approaches to computational problems from a wide variety of  fields.

 \section{Acknowledgements} 
 We are very grateful to Valentine Kabanets, Todd Wareham and Russell Impagliazzo for numerous discussions and suggestions, and to Venkat Guruswami for telling us about then-unpublished work of Sheldon and Young.

\bibliographystyle{alpha}
\bibliography{strapp}

\end{document}